\documentclass[sigconf]{acmart}

\usepackage{booktabs} 
\usepackage{lipsum}
\usepackage{multirow}
\usepackage{array}
\usepackage{courier}
\usepackage{mathtools}
\usepackage{subcaption}
\usepackage{amsfonts}
\usepackage{mathtools}

\DeclarePairedDelimiter{\abs}{\lvert}{\rvert}
\DeclarePairedDelimiter\norm{\lVert}{\rVert}%

\begin{document}
\copyrightyear{2018}
\acmYear{2018} 
\setcopyright{iw3c2w3}
\acmConference[WWW 2018]{The 2018 Web Conference}{April 23--27, 2018}{Lyon, France}
\acmBooktitle{WWW 2018: The 2018 Web Conference, April 23--27, 2018, Lyon, France}
\acmPrice{}
\acmDOI{10.1145/3178876.3186154}
\acmISBN{978-1-4503-5639-8}

\title{Latent Relational Metric Learning via Memory-based Attention for Collaborative Ranking}

\author{Yi Tay}
\affiliation{%
  \institution{Nanyang Technological University Singapore}
}
\email{ytay017@e.ntu.edu.sg}

\author{Luu Anh Tuan}
\affiliation{%
 \institution{Institute for Infocomm Research Singapore}
}
\email{at.luu@i2r.a-star.edu.sg}

\author{Siu Cheung Hui}
\affiliation{%
  \institution{Nanyang Technological University Singapore}
}
\email{asschui@ntu.edu.sg}

\begin{abstract}
This paper proposes a new neural architecture for collaborative ranking with implicit feedback. Our model, LRML (\textit{Latent Relational Metric Learning}) is a novel metric learning approach for recommendation. More specifically, instead of simple push-pull mechanisms between user and item pairs, we propose to learn latent relations that describe each user item interaction. This helps to alleviate the potential geometric inflexibility of existing metric learing approaches. This enables not only better performance but also a greater extent of modeling capability, allowing our model to scale to a larger number of interactions. In order to do so, we employ a augmented memory module and learn to attend over these memory blocks to construct latent relations. The memory-based attention module is controlled by the user-item interaction, making the learned relation vector specific to each user-item pair. Hence, this can be interpreted as learning an exclusive and optimal relational translation for each user-item interaction. The proposed architecture demonstrates the state-of-the-art performance across multiple recommendation benchmarks. LRML outperforms other metric learning models by $6\%-7.5\%$ in terms of Hits@10 and nDCG@10 on large datasets such as Netflix and MovieLens20M. Moreover, qualitative studies also demonstrate evidence that our proposed model is able to infer and encode explicit sentiment, temporal and attribute information despite being only trained on implicit feedback. As such, this ascertains the ability of LRML to uncover hidden relational structure within implicit datasets. 

\end{abstract}


\keywords{Collaborative Filtering; Recommender Systems; Neural Networks}

\maketitle

\section{Introduction}

The modern age is a world of information overload. The explosion of information, also referred to as the era of big data, is a huge motivator for the research and development of practical recommender systems. Generally, the key problem that these systems are aiming to solve is the inevitable conundrum of `too much content, too little time' that is commonly faced by users. After all, there are easily million of movies, thousands of songs and hundreds of books to choose from at any given time. An effective recommender system ameliorates this problem by delivering the most relevant content to the user. 

Our work is targeted at recommender systems that operate on implicit data (e.g., clicks, likes, bookmarks) and are known as collaborative filtering (CF) systems \cite{DBLP:conf/www/SarwarKKR01}. In this setting, Matrix Factorization (MF) remains as one of the most popular baselines which has inspired a considerable number of variations \cite{DBLP:conf/sigir/HeZKC16,DBLP:conf/uai/RendleFGS09,He:2017:NCF:3038912.3052569,DBLP:conf/kdd/Koren08}. The general idea of MF is as follows: Users and items are represented as a matrix and subsequently factorized into latent components which can also be interpreted as modeling the relationships between users and items using the inner product. As such, this allows missing values to be inferred which provides an approximate solution to the recommendation problem.

Recently, Hseih et al. \cite{DBLP:conf/www/HsiehYCLBE17} revealed the potential implications pertaining to the usage of inner product to model user-item relationships. Their argument is constructed upon the fact that inner product violates the triangle inequality which is essential to model the fine-grained preferences of users. Instead, the authors proposed a metric-based learning scheme that minimizes the distance between user and item vectors ($p$ and $q$) of positive interactions. Simultaneously, this also learns user-user similarity and item-item similarity in vector space. As evidence to their assertions, their proposed algorithm, the collaborative metric learning (CML) algorithm \cite{DBLP:conf/www/HsiehYCLBE17} demonstrates highly competitive performance on many benchmark datasets. 

Despite the success of CML, it faces several weaknesses. Firstly, the scoring function of CML is clearly \textit{geometrically restrictive}. Given a user-item interaction, CML tries to fit the pair into the same point in vector space. Considering the \textit{many-to-many} nature of the collaborative ranking problem, enforcing a good fit in vector space can be really challenging from a geometric perspective especially since the optimal point of each user and item is now a single point in vector space. Intuitively, this tries to fit a user and all his interacted items onto the same point, i.e., geometrically congestive and inflexible. While it is possible to learn user-user and item-item similarity clusters, this comes at the expense of precision and accuracy in ranking problems especially pertaining to large datasets whereby there can be millions of interactions. Secondly and by taking a more theoretically grounded angle, CML is an \textit{ill-posed algebraic system} \cite{illposed} which further reinforces and aggravates the problem of geometric inflexibility. A proof and more details are described in the related work section. 

In this work, we propose a flexible and adaptive metric learning algorithm for collaborative filtering and ranking. Our model, LRML (Latent Relational Metric Learning) learns adaptive relation vectors between user and item interactions, finding an optimal translation vector between each interaction pair. Needless to say, our work is highly inspired by recent advances in NLP which include the highly celebrated word embeddings \cite{DBLP:conf/nips/MikolovSCCD13} and knowledge graph embeddings \cite{DBLP:conf/nips/BordesUGWY13,DBLP:conf/aaai/LinLSLZ15,DBLP:conf/aaai/TayLH17} which popularized the concept of semantic translation in vector space. In our proposed approach, we assume that there exist a latent relational structure within the implicit interaction data and therefore, we aim to model the latent relationships between users and items by inducing relation vectors. 

Overall, our key intuition can be described as follows: For each user and item interaction, we learn a vector $r$ that explains this relationship, i.e., the relation vector $r$ connects the user vector to the item vector. Ideally, this vector $r$ should capture the hidden semantics between each implicit interaction and is learned over an auxiliary memory module via a neural attention mechanism. The auxiliary memory module can be interpreted as a memory store of concepts in which, upon linear combination, constructs a relation vector. The content addressing of this memory module is user and item dependent, which ensures sufficient flexibility in geometric space. Apart from the clear benefits of an interpretable attention module, LRML can also be considered as an improvement to the CML algorithm \cite{DBLP:conf/www/HsiehYCLBE17}. Our approach solves the geometric inflexibility problem by means of adaptive (user-item specific) translations in vector space. This allows for a greater extent of flexibility and modeling capability in metric space which enables our model to scale to larger datasets with easily millions of interactions. 

\subsection{Our Contributions}

Motivated by the success of deep learning, both generally and in the field of recommender systems, our ideas are materialized in the form of a neural network architecture that leverages the recent advancements of neural attention mechanisms and augmented memory modules \cite{DBLP:conf/nips/SukhbaatarSWF15}. Overall, the prime contributions of this paper are:

\begin{itemize}
\item We present \textsc{LRML} (Latent Relational Metric Learning), a novel, end-to-end neural network architecture for collaborative filtering and ranking on implicit interaction data. For the first time, we adopt \textit{user and item specific} latent relation vectors to model the relationship between user-item interactions. 
\item We propose a novel \emph{Latent Relational Attentive Memory} (LRAM) module in order to generate the latent relation vectors. The LRAM module provides improvements in terms of flexibility and modeling capability of the algorithm. Moreover, the neural attention also gives greater insight and interpretability of the model. 
\item We evaluate our proposed \textsc{LRML} on \textbf{ten} publicly available benchmark datasets. This includes large, web-scale datasets like Netflix Prize and MovieLens20M. Our proposed approach demonstrates highly competitive results on all datasets, outperforming not only CML but many other strong baselines such as NeuMF \cite{He:2017:NCF:3038912.3052569}. Moreover, on large datasets, we obtain $6\%-7.5\%$ gain in performance over CML and other models.   
\item We performed extensive qualitative analysis. Upon inspection of the attention weights, our proposed \textsc{LRML} is capable of inferring explicit information such as ratings (e.g., 1-5 stars), temporal and item attribute information despite being only trained on implicit binary data. This ascertains the capability of \textsc{LRML} in unraveling hidden latent structure within seemingly non-relational datasets. 
\end{itemize}

\section{Background}
Our work is concerned with collaborative filtering\footnote{In this paper, we use the terms collaborative filtering and collaborative ranking interchangeably.} with implicit feedback. We first formulate the problem and discuss the existing algorithms that are aimed at solving this problem. Then, we elaborate on the potential weaknesses of the collaborative metric learning algorithm.

\subsection{Implicit Collaborative Filtering}
The task of implicit collaborative filtering is concerned with learning via implicit interaction data, e.g., clicks, bookmarks, likes, etc. Let \textbf{P} be the set of all users and \textbf{Q} be the set of all items. The problem of implicit CF can be described as follows:
\begin{equation}
y_{ui} =
\begin{cases}
1 \:, \text{if interaction $<$user,item$>$ exists} \\
0 \: , \text{otherwise}\\
\end{cases}
 \end{equation}
where $\textbf{Y} \in \mathbb{R}^{\abs{P} \times \abs{Q}}$ is the user-item interaction matrix. Implicit CF models the interaction of users and items and on that note, it is good to bear in mind that a value of $0$ does not necessarily imply negative feedback. In most cases, the user is unaware of the existence of the item which forms the cornerstone of the recommendation problem, i.e., estimating the scores of the unobserved entries in \textbf{Y}. Across the past decade, Matrix Factorization (MF) techniques are highly popular algorithms for collaborative filtering and have spurred on a huge number of variations \cite{He:2017:NCF:3038912.3052569,DBLP:conf/sigir/HeZKC16}. Since MF does not belong to the core focus of our work, we omit the technical descriptions of MF for the sake of brevity and refer interested readers to \cite{He:2017:NCF:3038912.3052569,DBLP:conf/www/HsiehYCLBE17} for more details. 

\subsection{Collaborative Metric Learning (CML)}
\label{sec:cml}
CML \cite{DBLP:conf/www/HsiehYCLBE17} is a recently incepted algorithm for CF and has, despite its simplicity, demonstrated highly competitive performance on several benchmarks. The key intuition is that CML operates in metric space, i.e., it minimizes the distance between each user-item interaction in Euclidean space. The scoring function of CML is defined as:
\begin{equation}
s(p,q) = \norm{\:p-q\:}^{2}_{2}
\end{equation}
where $p,q$ are the user and item vectors respectively. CML learns via a pairwise hinge loss, which is reminiscent of the Bayesian Personalized Ranking (BPR) \cite{DBLP:conf/uai/RendleFGS09}. CML obeys the triangle inequality which, according to the authors, is a prerequisite for fine-grained fitting of users and items in vector space. 

CML, however, is not without flaws. As mentioned, the scoring function of CML is \textit{geometrically restrictive} since the objective function tries to fit each user-item pair into the same point in vector space. Unfortunately, this intrinsic geometric inflexibility causes adverse repercussions when the dataset is large or dense since CML tries to force all of a user's item interactions onto the same point. Secondly and by taking a more theoretically grounded angle, we show that CML is an \textit{ill-posed algebraic system} \cite{illposed} which further reinforces and aggravates the problem of geometric inflexibility. The following proof elaborates on this issue. 

\begin{theorem}
The objective function of CML: $s(p,q) = \norm{p-q}^{2}_{2}$ can be considered as an ill-posed algebraic system when there is a large number of interactions. 
\end{theorem}

\begin{proof}
Let $d$ be the dimensions of vectors $p$ and $q$. From an algebraic perspective, each user-item interaction can be regarded as the equation $p-q=0$. By considering $p_i - q_i$, where $i$ is the index of the vectors $p$ and $q$, the number of equations for each interaction is $d$. Let $N$ be the total number of interactions, the total number of equations is therefore $N \times d$. On the other hand, the number of free variables is only $(\abs{P} + \abs{Q}) \times d$. Since $N \ggg d(\abs{P} + \abs{Q})$ in most settings, 
CML is an ill-posed algebraic system. 
\end{proof}

 Since it is not uncommon for implicit recommendation datasets to contain millions of interactions while having significantly much lesser unique items and users, we can consider, from a mathematical perspective, that CML proposes an ill-posed algebraic system. This introduces instability when training and optimizing the objective function of CML. 

\subsection{Translating in Vector Space}

Our proposed approach, \textsc{LRML}, ameliorates the flaws of CML by means of adaptive translation. Since our adaptive translation is learned as a weighted representation (over an augmented memory via neural attention), this introduces an extremely large number of possibilities for the user and item vectors to be translated in vector space. More specifically, the attention vector (learned via a softmax function) learns a continious weighted representation of the augmented memory. As such, this significantly expands the flexibility of the metric learning algorithm. In LRML, the user vector is now adaptively translated based on the target item (and vice versa). As such, this allows \textsc{LRML} to avoid the above-mentioned flaws of CML, and enables more precise and fine-grained fitting in vector space. 

Translating in vector space takes inspiration from NLP and in particular, reasoning over semantics (knowledge graphs). In this area, a highly seminal work by Bordes et al. (TransE) \cite{DBLP:conf/nips/BordesUGWY13} proposed translations in vector space to model the relationships between entities in a knowledge graph. Word embeddings \cite{DBLP:conf/nips/MikolovSCCD13} are also known to exhibit \textit{semantic translation} in vector space whereby the relationships between two words can be explained by a relation vector. The domain of CF that models users and items, and represents them as an interaction matrix is highly related to graph and network embeddings \cite{DBLP:conf/kdd/PerozziAS14,DBLP:conf/www/TangQWZYM15,Ma:2018:MNE:3159652.3159680}. 

To the best of our knowledge, our work is the first work that extends the 2D structure of user-item CF into a 3D structure by assuming a latent relational (3D) structure. Intuitively, this can be also interpreted as inducing a latent knowledge graph from the user-item interaction graph. 

Figure \ref{fig:vector} depicts the key difference between \textsc{LRML} and CML - while CML tries to place user and item into the same spot in vector space, \textsc{LRML} learns to fit user and item with adaptive, trainable latent vectors. 
 More specifically, \textsc{LRML} learns an optimal translation between each user-item interaction. Recall in Section \ref{sec:cml}, we have previously established that CML suffers from instability (due to being an ill-posed algebraic system) along with geometric inflexibility, i.e., the push-pull effects from too many interactions. In order to alleviate this weakness, our proposed approach adopts attentive and adaptive user-item specific translations that benefit from the vast number of possibilities of learning weighted (linearly combined) representations.

 Finally, we note that another translation-based recommendation model, TransRec was recently proposed by He et al. \cite{He:2017:TR:3109859.3109882} in which the authors proposed to use translations to model sequential data. While TransRec also utilizes the translation principle, LRML is a completely different model. Firstly, TransRec learns translations for sequential recommendation, e..g, the 2nd item a user interacts with is represented by a translation of the first. Secondly, the overall goals of LRML is different, i.e., LRML utilizes translations for flexible and adaptive metric learning. Thirdly, LRML uses neural attention to learn latent relations, which is also an feature that is absent in TransRec.

 \begin{figure}[H]
\centering
\begin{subfigure}{0.25\textwidth}
  \centering
  \includegraphics[width=0.9\linewidth]{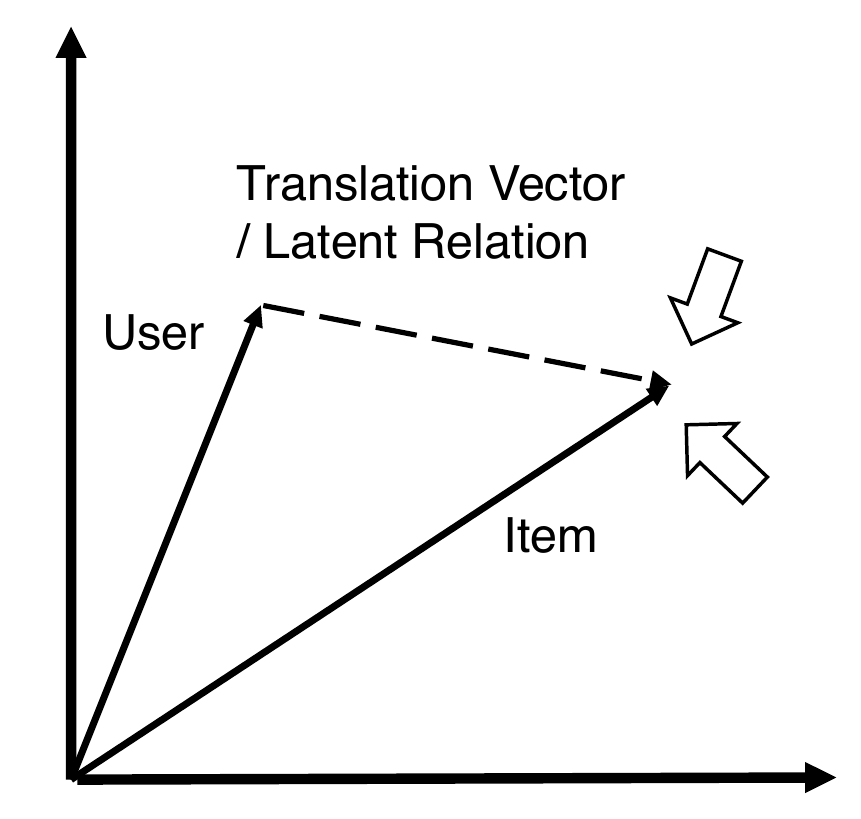}

  \caption{\textsc{LRML}}

  \label{fig:sub1}
\end{subfigure}%
\begin{subfigure}{0.25\textwidth}
  \centering
  \includegraphics[width=0.9\linewidth]{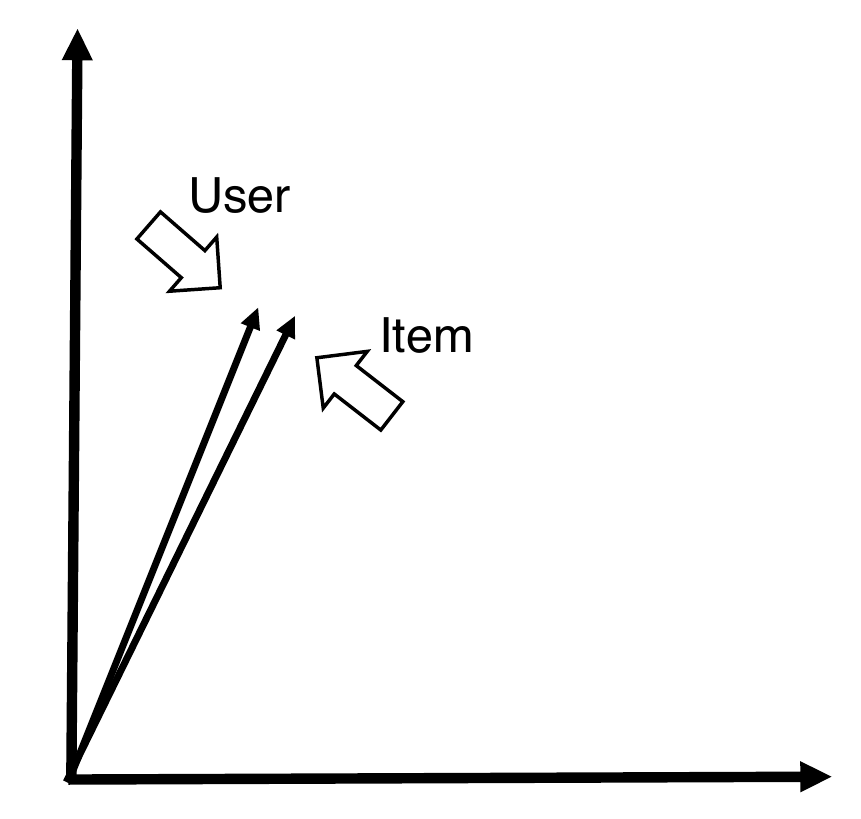}
    \caption{CML}
  \label{fig:sub2}
\end{subfigure}
\caption{Geometric Comparisons of Latent Relational Metric Learning (LRML) and CML (Collaborative Metric Learning) for Modeling User-Item Relationships in Metric Space. }
\label{fig:vector}
\end{figure}

\subsection{Deep Learning}
In this section, we provide some preliminaries about deep learning for recommendation. 
\subsubsection{Deep Learning for Recommendation}
 In the recent years we can easily observe the emerging numbers of neural network models that have been designed for a diverse range of recommendation tasks. Notably, Recurrent neural networks \cite{DBLP:conf/wsdm/WuABSJ17} and convolutional neural networks \cite{Tang:2018:PTS:3159652.3159656} have been exploited for sequence aware recommendations. There is also an emerging line of work focusing on representation learning using reviews, e.g., Deep Co-operative Networks (DeepCoNN) \cite{DBLP:conf/wsdm/ZhengNY17}. A recent work, the Multi-Pointer Co-Attention Networks \cite{DBLP:journals/corr/abs-1801-09251} is the state-of-the-art review-based CF model that uses pointer-based attention for representation learning. Autoencoder based models \cite{DBLP:conf/kdd/LiS17,Zhang:2017:AEH:3077136.3080689} have also been proposed for CF. In the more closely related domain of collaborative filtering on implicit feedback, Neural Matrix Factorization (NeuMF) \cite{He:2017:NCF:3038912.3052569} is a recent state-of-the-art deep learning model that learns the interaction function between user and item using deep neural networks. NeuMF is a combined framework that concatenates the inner-product-based MF with a multi-layered perceptron (MLP). A comprehensive review can be found at \cite{zhang2017deep}.

 \subsubsection{Neural Attention} 
 Our work borrows inspiration from the recent advances in deep learning. Specifically, \textsc{LRML} uses a neural attention mechanism over an augmented memory module to generate latent vectors. Neural attention mechanisms are popular in the fields of computer vision \cite{DBLP:conf/nips/MnihHGK14,DBLP:conf/icml/XuBKCCSZB15} and NLP \cite{rocktaschel2015reasoning,luong2015effective,1712.05403,DBLP:conf/cikm/PhanSTHL17,DBLP:journals/corr/abs-1801-00102} and are known to improve performance and interpretability of deep learning models. The key idea of attention is to learn a weighted representation across multiple samples (or embeddings), reducing noise and selecting more informative features for the final prediction. Attention operates using a \textit{softmax} function, which converts the attention vector into a probability distribution. Subsequently, this vector is then used to learn a weighted sum of a sequence of vectors. 

 Notably, attention mechanisms have been also recently adopted for collaborative filtering problems particularly for content-based recommendations such as multimedia recommendation \cite{DBLP:conf/sigir/ChenZ0NLC17}. However, the novelty of our model lies in the difference whereby our model adopts neural attention to generate latent relation vectors over an augmented memory module. This is fundamentally different from content-based attention models which learn to attend over features and learn to predict. While the key idea of attentive selection is similar, the goal of our model is to find hidden relational structure by leveraging attention mechanisms. Moreover, the inner mechanism of our proposed LRAM is highly reminiscent of end-to-end memory networks \cite{DBLP:conf/nips/SukhbaatarSWF15} and key-value memory networks \cite{miller2016key} which are the competitive models for question answering, machine comprehension and aspect-aware sentiment analysis \cite{DBLP:conf/cikm/TayTH17}.

\section{Our Proposed Model}

In this section, we introduce \textsc{LRML}, our novel deep learning recommendation architecture. The overall model architecture is described in Figure \ref{fig:transrec}. \textsc{LRML} aims to model user and item pairs using relation vectors. This is what we refer to as the translational principle, i.e., $p+r \approx q$. Note that the relation vector $r$ is what separates our model from simple metric learning approaches like CML which operate via $p \approx q$. Let us begin with a simple high-level overview of our model:
\begin{enumerate}
\item Users and items are converted to dense vector representations using an Embedding Layer (a look-up layer). $p$ and $q$ are the user and item vectors respectively. 
\item Given $p$ and $q$, a relation vector $r$ is generated using a neural attention mechanism over an augmented memory matrix $\textbf{M}$. The relation vector, $r$, is a weighted representation over a trainable LRAM module. $r$ is dependent on user and item, and is learned to best explain the relationship between user and item. 
\item Our model optimizes for $\norm{\: p+r-q \:} \approx 0$ using pairwise ranking (hinge loss) and negative sampling. 
\end{enumerate}

\begin{figure}[ht]
\begin{center}
\includegraphics[width=0.46\textwidth]{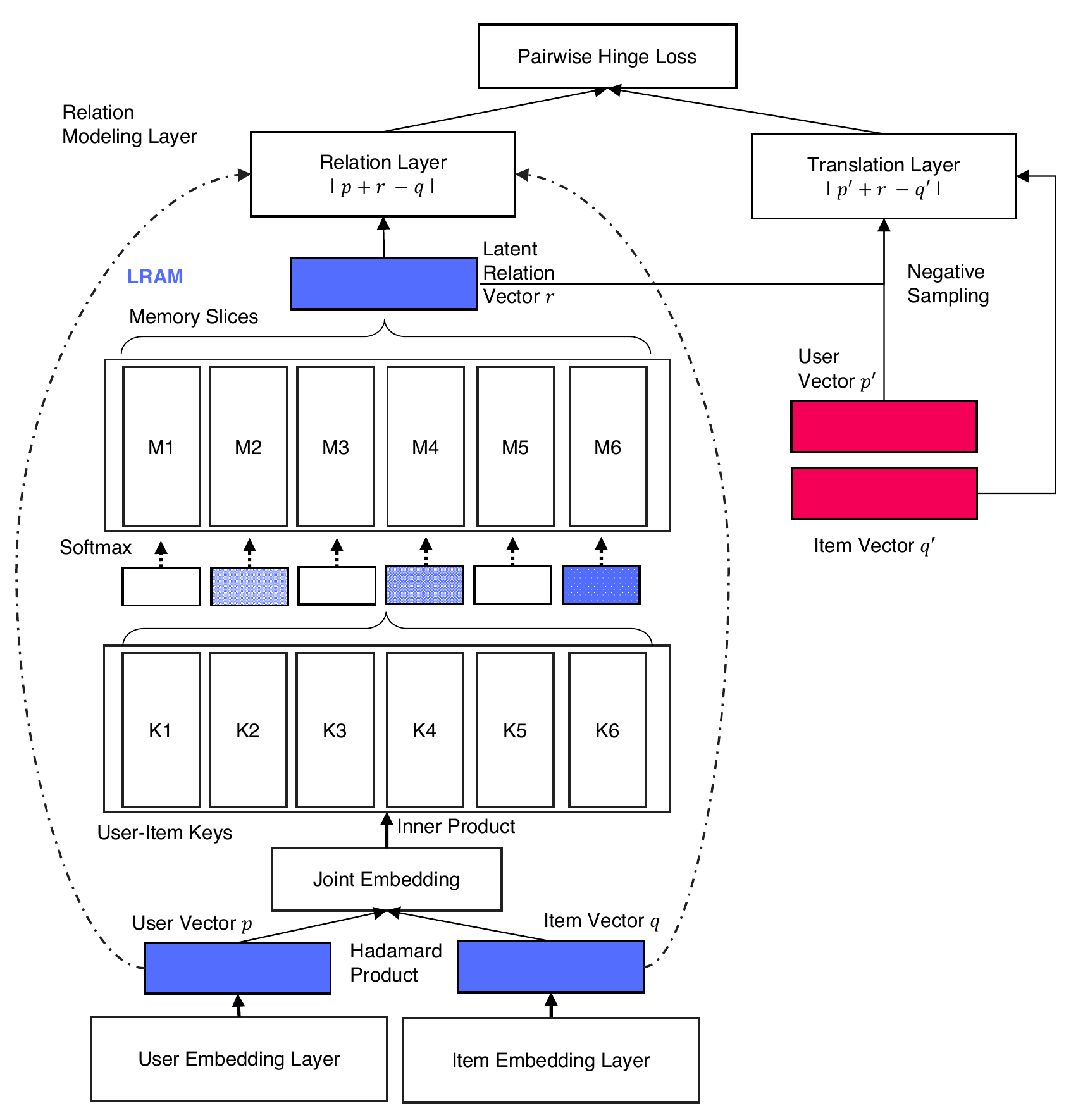}
\caption{Illustration of our proposed \textsc{LRML} architecture, an end-to-end differentiable neural architecture. LRML is characterized by its key-addressed LRAM module which learns user-item specific relation vectors. The size of the memory \textbf{N}=6 slices in this example.}
\label{fig:transrec}
\end{center}
\end{figure}

\subsection{Embedding Layer}
\textsc{LRML} accepts a user-item pair \emph{$<$ user,item $>$} as an input. Inputs of users and items are 
represented as vectors encoded via one-hot encoding corresponding to a unique index key belonging to each user and item. At the 
embedding layer, this one-hot encoded vector is converted into a low-dimensional real-valued dense vector representation. 
In order to do so, this one hot vector is multiplied with the embedding matrices $\textbf{P} \in \mathbb{R}^{d \times |U|}$
and $\textbf{Q} \in \mathbb{R}^{d \times |I|}$ which store the user and item embeddings respectively. $d$ is the dimensionality of
the user and item embeddings while $|U|$ and $|I|$ are the total number of users and items respectively. The output of this layer is a pair
of embeddings $<\vec{p},\vec{q}>$ which are the user and item embeddings respectively.

\subsection{LRAM - Latent Relational Attentive Memory Module}
One of the primary goals of our \textsc{LRML} model is to induce latent relations between user-item pairs. However, explicit semantic relations between user-item pairs are not available in traditional CF. As such, we introduce the \emph{Latent Relational Attentive Memory} (LRAM) module. The LRAM module is a centralized memory store in which latent relations are built upon. The memory matrix of the LRAM module is represented as $\textbf{M} \in \mathbb{R}^{N \times d}$ where $d$ is the dimensionality of the user-item embeddings and $N$ is a user-specified hyperparameter that controls the expressiveness and capacity of the LRAM module. In matrix \textbf{M}, we refer to each row slice $m_{i} \in \mathbb{R}^{d}$ as a memory slice. The input to LRAM is a user-item pair $<p,q>$. The LRAM module returns the vector $r$ of equal dimensionality as p and q. 

\subsubsection{Joint User-Item Embedding}
Given the user-item pair, $<\vec{p},\vec{q}>$, the LRAM module first applies the following steps to learn a joint embedding of users and items:
\begin{equation}
s = p \: \odot \: q
\end{equation}
where $\odot$ is simply the Hadamard product (or element-wise multiplication). The generated vector $s \in \mathbb{R}^{d}$ is of the same dimension of $p$ and $q$. Note that while other functions such as the multi-layered perceptron $MLP(p,q)$ are also viable, we found that a simple Hadamard product performs better. 

\subsubsection{User-Item Key Addressing}
Next, using the joint user-item embedding, we aim to learn an attention vector $a$. The attention vector is learned from $\textbf{K} \in \mathbb{R}^{N \times d}$ which we refer to as the key matrix. Each element of the attention vector $a$ can be defined as:
\begin{equation}
a_i = s^{T} k_i
\end{equation}
where $k_i \in \textbf{K} \in \mathbb{R}^{N \times d}$ and the generated vector $a \in \mathbb{R}^{d}$ is of the same dimensions of $p$, $q$ and $s$. In order to normalize $a$ to a probability distribution, we can simply use the Softmax function:

\begin{equation}
 Softmax(a_i) = \frac{e^{a_i}}{\sum_{j}e^{a_j}}.
\end{equation}
Since our attention mechanism utilizes the softmax function, it ensures that our network is end-to-end differentiable. 
\subsubsection{Generating Latent Relations via Memory-based Attention}
Finally, in order to generate the latent relation vector $r$, we use the attention vector $a$ to generate a weighted representation of $\textbf{M}$, i.e., adaptively selecting relevant pieces of information from the memory matrix $\textbf{M}$.
\begin{equation}
r = \sum_{i} a_i m_i
\end{equation}
The output of the LRAM module is a user and item specific latent relation vector $r$. The latent relation vector is a weighted representation of $\textbf{M}$. Intuitively, the memory matrix $\textbf{M}$ can be interpreted as a store of conceptual building blocks that can be used to describe the relationships between users and items. The mechanism design of the LRAM module is inspired by Memory Networks and can also be interpreted as neural attentions which give our model improved interpretability. Note that the LRAM module is part of \textsc{LRML} and is trained end-to-end. Finally, the total number of parameters added by the LRAM module is merely $2 \times N \times d$ parameters and since typically we set $N<100$ in our experiments, the parameter cost incurred by the LRAM module is negligible.   
\subsection{Optimization and Learning}
In this section, we introduce the final layer of the network, the objective function and the regularization employed in our training scheme. \textsc{LRML} is end-to-end differentiable since it utilizes soft attention over the LRAM module. As such, we are able to simply train it via stochastic gradient descent (SGD) methods. 
\subsubsection{Relational Modeling Layer}
For each user-item pair $p$ and $q$, the scoring function is defined as:
\begin{equation}
s(p,q) = || \: p + r - q \: ||^{2}_{2} 
\end{equation}
where $r$ is the latent relation vector constructed from the LRAM module and $||.||^{2}_{2}$ is essentially the L2 norm of the vector $p+r-q$. Intuitively, this score function penalizes any deviation of (p+r) from the vector q.
\subsubsection{Objective Function}
\textsc{LRML} adopts the pairwise ranking loss or hinge loss for optimization. For each positive user-item pair $<p,q>$, we sample a corrupted pair which we denote as $<p',q'>$. Similar to the positive example, the corrupted pair of user and item goes through the same user and item embedding layer respectively. The pairwise ranking / hinge loss is defined as follows:
\begin{equation}
L =  \sum_{(p,q) \in \Delta} \sum_{(p',q') \not\in \Delta} max(0,s(p,q) + \lambda - s(p',q'))
\end{equation}
where $\Delta$ is the set of all user-item pairs, $\lambda$ is the margin that separates the golden pairs and corrupted samples.$\: max(0,x)$ is also known as the \textit{relu} function. Note that we use the same (generated) latent relation vector for the negative example. This is motivated by our early empirical results whereby performance was much better over generating a separate relation vector for the negative example. 
\subsubsection{Regularization}
Finally, we apply regularization by normalizing all user and item embeddings to be constrained within the Euclidean ball. At the end of each mini-batch, we apply a constraint of $\norm{p_*}_2 \leq 1$ and $\norm{q_*}_2 \leq 1$ for regularization and preventing overfitting. In order to enforce this, we can manually project all embeddings to the unit ball either at the beginning or after each training iteration.

\section{Performance Evaluation}
In this section, we evaluate our proposed \textsc{LRML} against other state-of-the-art algorithms. Our experimental evaluation is designed to answer several research questions (\textbf{RQs}).
\begin{itemize}
\item \textbf{RQ1:} Does LRML outperform other baselines and state-of-the-art methods for collaborative ranking? 
\item \textbf{RQ2:} How does the relative performance of LRML and CML differ across different dataset sizes?
\item \textbf{RQ3:} What is the scalability and runtime of LRML compared to other baselines?
\item \textbf{RQ4:} What is the LRAM module learning? Are we able to derive qualitative insights about the inner workings of LRML?
\item \textbf{RQ5:} What do the relation vectors represent? Are they meaningful?
\end{itemize}
\subsection{Datasets}

In the spirit of experimental rigor, we conduct our evaluation across a wide spectrum of datasets. 


\begin{itemize}
\item \textbf{Netflix Prize} - Since the entire Netflix Prize dataset is extremely large, we construct a subset of the famous Netflix Prize dataset. Specifially, we only considered movie-item ratings from the year 2005 and filtered users who had less than $100$ interactions. 
\item \textbf{MovieLens} - A widely adopted benchmark dataset\footnote{https://grouplens.org/datasets/movielens/} for collaborative filtering in the application domain of recommending movies to users. Specifically, we use two configurations of this benchmark dataset, namely MovieLens1M and MovieLens20M \cite{DBLP:journals/tiis/HarperK16}. 
\item \textbf{IMDb} - A movie recommendation dataset obtained from IMDb that was introduced in \cite{DBLP:conf/kdd/DiaoQWSJW14}.
\item \textbf{LastFM} - This dataset contains social networking, tagging, and music artist listening information 
    from Last.fm online music system\footnote{http://last.fm}. 
\item \textbf{Books} - This is a book recommendation dataset that was used in \cite{DBLP:conf/www/ZieglerMKL05}. 
\item \textbf{Delicious} - This dataset contains social networking, bookmarking, and tagging information 
    from a set of 2K users from Delicious Social Bookmarking System\footnote{http://www.delicious.com}. This dataset, along with the lastFM dataset, originated from the Hetrec 2011 workshop\footnote{https://grouplens.org/datasets/hetrec-2011/}.
\item \textbf{Meetup} - An event-based social network\footnote{https://www.meetup.com/}. We use the datasets provided by \cite{DBLP:conf/icde/PhamLCZ15} which include event-user pairs from NYC.
\item \textbf{Twitter} - This is a check-in dataset constructed by \cite{DBLP:conf/kdd/YuanCMSM13} which contains users and their check-ins. There are two subsets of this dataset, namely Twitter (WW) and Twitter (USA).  
\end{itemize}
In total, we evaluate our proposed algorithm on \textbf{ten} different datasets with diverse sizes and interaction densities, i.e., the percentage of non-zero values in the user-item interaction matrix. For all datasets, with the exception of the Netflix Prize dataset, we ensured that each user has at least 20 interactions. The statistics of all datasets are reported in Table \ref{tab:dataset}.

\begin{table}[htbp]
  \centering
    \begin{tabular}{c|cccc}
    \hline
    \textbf{Dataset} & \textbf{Interactions} & \textbf{$\#$Users} & \textbf{$\#$Items} & \textbf{$\%$ Density} \\
    \hline
    Netflix Prize & 44M   & 75K   & 13K   & 4.5 \\
    MovieLens20M & 16M   & 53K   & 27K   & 1.1 \\
    MovieLens1M & 1M    & 6K    & 4K    & 4.2 \\
    IMDb  & 117K  & 0.8K  & 114K  & 0.13 \\
    LastFM & 92K   & 1.9K  & 175K  & 0.28 \\
    Books & 285K  & 7.4K  & 291K  & 0.01 \\
    Delicious & 43K   & 1.7K  & 69K   & 0.36 \\
    Meetup & 11K   & 2.6K  & 16K   & 0.26 \\
    Twitter (WW) & 18K   & 4K    & 36K   & 0.01 \\
    Twitter (USA) & 171K  & 4K    & 36K   & 0.12 \\
    \hline
    \end{tabular}%
      \caption{Statistics of all datasets used in our experimental evaluation. }
  \label{tab:dataset}%
\end{table}%

\begin{table*}[htbp]
  \centering
  \small
    \begin{tabular}{c|cccccccccc|cc}
    \midrule
          & \multicolumn{2}{c}{BPR} & \multicolumn{2}{c}{MLP} & \multicolumn{2}{c}{MF} & \multicolumn{2}{c}{NEUMF}
           & \multicolumn{2}{c}{CML} & \multicolumn{2}{c}{\textsc{LRML}} \\
           \midrule
          & H@10  & nDCG@10 & H@10  & nDCG@10 & H@10  & nDCG@10 & H@10  & nDCG@10 & H@10  & nDCG@10 & H@10  & nDCG@10 \\
          \midrule
    Netflix & \underline{48.67} & \underline{31.97} & 33.77 & 22.34 & 47.07 & 30.25 & 32.27 & 22.59 & 46.12 & 29.48 & \textbf{53.71} & \textbf{35.78} \\
    MovieLens20M & 69.68 & 46.68 & 75.81 & \underline{54.38} & 72.98 & 49.01 & 75.82 & 54.37 & \underline{77.64} & 53.01 & \textbf{84.47} & \textbf{61.52} \\
    MovieLens1M & \underline{72.37} & 53.33 & 48.59 & 33.11 & 68.87 & 49.17 & 68.61 & 50.65 & 72.16 & \underline{54.13} & \textbf{73.97} & \textbf{54.53} \\
    IMDb  & 4.62  & 4.23  & 4.11  & 3.79  & 5.26  & 4.89  & 4.87  & 4.55  & \underline{9.47}  & \underline{7.16}  & \textbf{11.92} & \textbf{8.45} \\
    LastFM & \underline{20.73} & \underline{13.58} & 7.36  & 3.75  & 18.17 & 12.02 & 14.89 & 9.61  & 19.75 & 12.03 & \textbf{21.71} & \textbf{14.38} \\
    Books & 22.07 & 16.13 & 12.89 & 10.03 & 15.61 & 10.75 & 12.54 & 7.65  & \underline{25.86} & \underline{18.70} & \textbf{26.72} & \textbf{19.43} \\
    Delicious & 78.50 & 77.78 & 77.05 & 73.80 & 78.91 & 78.09 & 78.79 & 78.11 & \underline{79.31} & \underline{78.43} & \textbf{80.31} & \textbf{79.01} \\
    Meetup & 44.91 & 36.08 & 31.33 & 23.19 & \underline{47.23} & \underline{38.29} & 32.76 & 25.79 & 47.04 & 36.64 & \textbf{50.19} & \textbf{40.48} \\
    Twitter (WW) & 76.39 & 75.27 & 53.33 & 35.68 & \underline{76.93} & 75.43 & \underline{76.66} & 74.86 & 75.86 & 74.72 & \textbf{78.92} & \textbf{77.17} \\
    Twitter (USA) & 75.88 & 75.04 & 77.91 & 76.23 & 76.47 & 75.62 & 70.75 & 69.79 & \underline{78.30} & \underline{76.50} & \textbf{79.36} & \textbf{77.85} \\
    \midrule
    \end{tabular}%
    \caption{Experimental results on ten benchmark datasets. Best performance is in boldface and second best is underlined. \textsc{LRML} achieves best performance on all datasets, outperforming many strong neural baselines. Improvement is much larger on large datasets such as Netflix Prize or MovieLens20M. }
  \label{tab:exp_results}%
\end{table*}%

\subsection{Baselines}
In this section, we introduce the key baselines for comparison against our proposed \textsc{LRML}.
\begin{itemize}

\item \textbf{Bayesian Personalized Ranking} (BPR) \cite{DBLP:conf/uai/RendleFGS09} is a strong CF baseline that minimizes $\sum_{i} \sum_{j,k} -\log \sigma (p_i^T q_j - p_i^T q_k)$ + $\lambda_v \norm{u_i}^2 + \lambda_q \norm{q_j}^2$, where $(p_i,q_j)$ is a positive interaction and $(p_i, q_k)$ is a negative sample.

\item \textbf{Matrix Factorization} (MF) is a standard baseline for CF that models the relationship between user and item using inner products. We use the generalized version from \cite{He:2017:NCF:3038912.3052569} which scores user item pairs with $s(p,q) = \sigma(h^{T} (p \odot q))$. 
\item \textbf{Multi-layered Perceptron} (MLP) is the baseline neural architecture proposed in \cite{He:2017:NCF:3038912.3052569} in which the authors proposed to use multiple layers of nonlinearities to model the relationships between users and items. 
\item \textbf{Neural Matrix Factorization} (NeuMF) \cite{He:2017:NCF:3038912.3052569} is the state-of-the-art unified framework combining MF with MLP. NeuMF concatenates the output of MF and MLP, and uses a regression layer to predict the user item rating. Note that NeuMF uses separate embedding representations of users and items for MF and MLP. 
\item \textbf{Collaborative Metric Learning (CML)} \cite{DBLP:conf/www/HsiehYCLBE17} can be considered as the baseline of our model which does not include relational translations between user and item vectors. 
\end{itemize}
Since CML and NeuMF have surpassed many other baselines such as WMF \cite{DBLP:conf/icdm/HuKV08}, eALS \cite{DBLP:conf/sigir/HeZKC16} and Factorization Machines \cite{DBLP:conf/icdm/Rendle10}, we do not further report them.  Additionally, for fair comparison and due to scalability issues, we do not use WARP \cite{DBLP:journals/ml/WestonBU10} for both CML and \textsc{LRML}.

\subsection{Evaluation Protocol and Metrics}
Our evaluation protocol follows He et al. \cite{He:2017:NCF:3038912.3052569} very closely. Similarly, we adopt the \textit{leave-one-out} evaluation protocol, i.e., the testing set comprises the last item of all users. If there are no timestamps available in the dataset (e.g., \textit{Delicious} and \textit{LastFM}), then the test sample is randomly sampled. A single item from each user is also sampled to form the development set. Since it is too time consuming to rank all items for every user, we randomly sampled $100$ items that have no interactions with the target user and ranked the test item with respect to these $100$ items. This is in concert with many works  \cite{DBLP:conf/www/BayerHKR17,DBLP:conf/sigir/HeZKC16,DBLP:conf/uai/RendleFGS09,He:2017:NCF:3038912.3052569}. Since our problem is essentially formulated as learning-to-rank, we judge the performance of our model based on the popular and widely adopted standard metrics used in information retrieval and recommender systems: \textbf{normalized discounted cumulative gain} (nDCG@10) \cite{DBLP:journals/tois/JarvelinK02} and \textbf{Hit Ratio} (H@10). Intuitively, the nDCG@10 metric is a position-aware ranking metric while H@10 metric simply considers whether the ground truth is ranked amongst the top 10 items. For more detailed explanations, we refer readers to \cite{He:2017:NCF:3038912.3052569}.

\subsection{Implementation Details}
We implemented all models in \textit{TensorFlow\footnote{https://www.tensorflow.org/}} on a Linux machine. For tuning the hyperparameters, we select the model that performs best on the development set based on the nDCG metric and report the result of that model on the test set. Model parameters are saved every $50$ epochs. All models are trained until convergence, i.e., if the performance (nDCG metric) on the development set does not improve after $50$ epochs. Models are trained for a maximum of 500 epochs. For large datasets like \textit{MovieLens20M} and \textit{Netflix Prize}, we stop the training at 100 epochs. The dimensionality of user and item embeddings $d$ is tuned amongst $\{20, 50,100\}$. The number of batches $B$ is tuned amongst $\{10,100,1000\}$. The minimum number of batches for NetflixPrize and MovieLens20M is $100$ in order to fit into the GPU RAM. We optimize all models using the Adam optimizer \cite{DBLP:journals/corr/KingmaB14}. The learning rate for all models are tuned amongst $\{0.01,0.005,0.001\}$. For models that minimize the hinge loss, the margin $\lambda$ is tuned amongst $\{0.1,0.2,0.25,0.5\}$. For NeuMF and MLP models, we follow the configuration and architecture proposed in He et al. \cite{He:2017:NCF:3038912.3052569}, i.e., 3 fully-connected layers with a pyramid architecture. However, for fair comparison of all models, we do not use pretrained MLP and MF models in the NeuMF model since this effectively acts as an ensemble classifier. For \textsc{LRML}, the number of memory slices in $\textbf{M}$ is tuned amongst $N=\{5,10,20,25,50,100\}$. For simplicity, each training instance is paired with only a single negative sample. All embeddings and parameters are normally initialized with a standard deviation of $0.01$.

For most datasets and baselines, we found that the following hyperparameters work well: learning rate$=0.001$, number of batches $B=10$ and $\lambda=0.2$. A larger embedding size always performs better, i.e., $d=100$. The size of LRAM is dataset dependent. We found that setting $N=20$ works well for most datasets (performance does not degrade going beyond 50 but does not improve either). However, we found that setting $N=100$ works better on large datasets such as \textit{Netflix Prize} and \textit{MovieLens20M}.  

\subsection{Experimental Results}

The empirical results of our proposed model and baselines on 10 benchmark datasets are reported in Table \ref{tab:exp_results}. Our proposed \textsc{LRML} performs extremely competitively on all datasets and obtain the best performance on both nDCG@10 and H@10 metrics on \textbf{all} datasets. This answers \textbf{RQ1}, showing that our proposed \textsc{LRML} is capable of effective collaborative ranking. Moreover, the ranking of many of the competitor baselines is fluctuating across datasets as we see the second best performance is scattered amongst different models. 

\subsubsection{Comparison against CML}
In general, \textsc{LRML} outperforms CML on all datasets on both H@10 and nDCG@10 metric. We would like to draw the reader's attention to the two datasets, namely \textit{Netflix Prize} and \textit{MovieLens20M} datasets in which \textsc{LRML} obtained a clear margin in performance gain over the competitor models. This ascertains our earlier claim about the flaws of CML (not being able to scale to large datasets) and empirically proves the advantages of our proposed approach. Specifically, \textsc{LRML} outperforms CML by performance gains about $7.5\%$ on \textit{MovieLens20M} and about $6\%$ on \textit{Netflix Prize} on the nDCG@10 metric. The performance gains on the hit ratio (H@10) metric is also similarly high. When the dataset is smaller, the performance gains are less distinct. For example, the performance gain in MovieLens20M is much larger than in MovieLens1M. The performance gains on smaller datasets range from a marginal $1\%-2\%$, (e.g., \textit{Books} and \textit{Delicious}) to reasonably large, e.g., $3\%-4\%$ on the \textit{Meetup} or \textit{Twitter (WW)} datasets. As such, the concluding findings pertaining to the comparison of \textsc{LRML} and CML can be drawn as follows: On large datasets, the performance gain of \textsc{LRML} over CML is large. However, on smaller datasets, \textsc{LRML} at least performs equally well or sometimes reasonably better. This answers \textbf{RQ2} on the effect of dataset size on relative performance of LRML and CML. Our experimental evidence shows that our proposed \textsc{LRML} is effective and ascertains our usage of adaptive translations in metric learning. 

\subsubsection{Comparison against Other Baselines}
Pertaining to the performance of the other baselines, we found that the performance of MF and BPR is extremely competitive, i.e., both MF and BPR outperform CML on several datasets. The performance of MLP, on the other hand, seem to perform reasonably well only on \textit{MovieLens20M} and performs horribly on most datasets. Note that we also tried a non-pyramid architecture but that did not improve the performance.  The performance of the model NeuMF (that combines MLPs with MF) is often better than vanilla MLP but falls short of MF in most cases. Notably, NeuMF performs reasonably well on \textit{MovieLens20M}, \textit{Netflix Prize} and \textit{MovieLens1M}. This could possibly mean that the usage of dual embedding spaces (one for MF and one for MLP) might be overfitting on the smaller datasets. 

\subsubsection{Comparison on Runtime}
Figure \ref{fig:runtime} reports the runtime (seconds taken to run a single epoch) of all models on \textit{Netflix Prize} and \textit{MovieLens20M}. We make several observations. First, the difference in runtime between \textsc{LRML} and CML is quite insignificant, i.e., \textsc{LRML} only spends $\approx 10s-15s$ extra per epoch which is only a $5\%-10\%$ increase in runtime on both datasets. On the other hand, it is still faster than models such as NeuMF and MLP. Notably, this is also contributed by the fact that MLP and NeuMF are point-wise models which do not pair negative samples with positive samples during training. Next, we also compare the runtime of \textsc{LRML} with different $N$ (LRAM size) values and found that there is only minimal observable difference in runtime with $N=50$ or $N=100$. This was probably made insignificant by the highly optimized GPU operations and also due to the fact that the size of the matrix-vector operations in LRAM is relatively small. To answer \textbf{RQ3}, we have shown that \textsc{LRML} only incurs a slight computation cost over CML. 

\begin{figure}[ht]
\begin{center}
\includegraphics[width=0.3\textwidth]{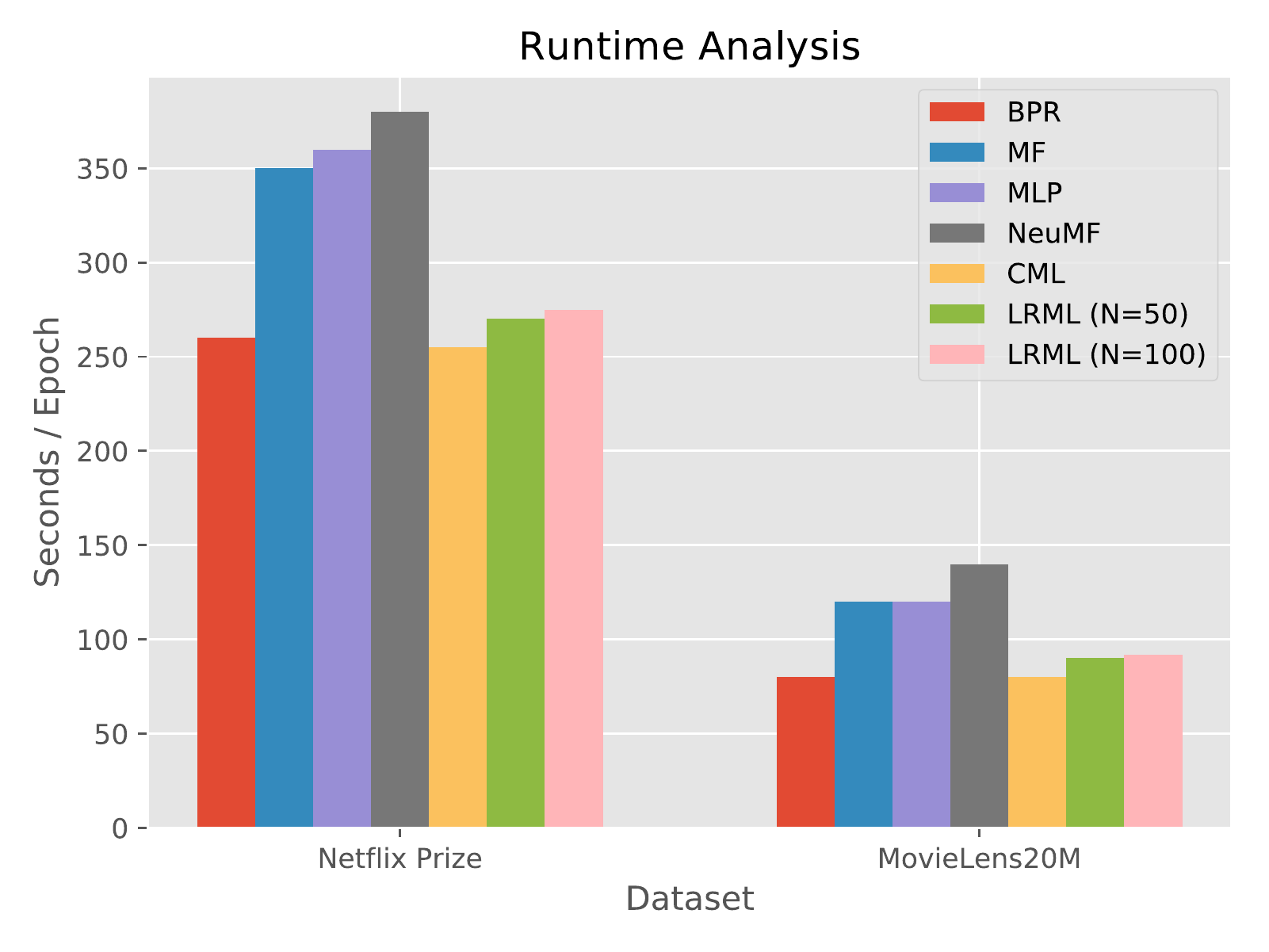}
\caption{Runtime (seconds/epoch) of all models on Netflix Prize and MovieLens20M. Experiments were run with batch size of 100 on a Nvidia P100 GPU. LRML only incurs a small computational cost over CML. \textit{(Best viewed in color.)}}
\label{fig:runtime}
\end{center}
\end{figure}

\section{Discussion and Analysis}
In this section, we derive qualitative insights regarding our proposed model. This section describes the discoveries that we have made while trying to understand and gain some intuition behind the performance of \textsc{LRML}.

\subsection{RQ4: What is the LRAM module learning?}
A key advantage to neural attention mechanisms is an improved interpretability since we are able to visualize the weighted importance of each memory slice with respect to any given attribute value $v$. This helps us to understand how the model is learning. Specifically, we investigate the attributes of \textit{explicit rating information} and \textit{explicit temporal information}, and show empirically via visualisation that the LRAM model learns to encode these attributes. Note that both attributes are not provided to our model at training time. In this experiment, the following steps were taken:
\begin{enumerate}
\item First, we categorized all user-item pairs $(p,q)$ according to the target attribute value $v$.
\item Using $(p,q)$ as an input, we generated the attention vector $a$ for each user-item pair. Recall that this attention vector\footnote{It is good to note that, at initialization, this attention vector looks at all memory slices \textit{almost} equally irregardless of $v$.}  is a probability distribution that depicts how much the model is looking at each memory slice of the LRAM module. 
\item For each attribute class $c_i \in v$, we take the mean attention vector for all user-item pairs in the category. 
\item We visualise the mean attribute vector of each attribute class to observe the correlation between attribute class and which memory slice \textsc{LRML} is looking at. 
\end{enumerate}

\subsubsection{LRAM Encodes Explicit Rating Information}
On datasets like MovieLens1M, explicit ratings (1-5 stars) exist but are not provided to \textsc{LRML}. Surprisingly, we empirically discovered that, despite being only trained on implicit interactions, explicit rating information is actually being encoded in LRAM. Figure \ref{attention_rating} shows the mean attention vector (i.e., $a$) for each rating class (1-5).

\begin{figure}[ht]
\begin{center}
\includegraphics[width=0.3\textwidth]{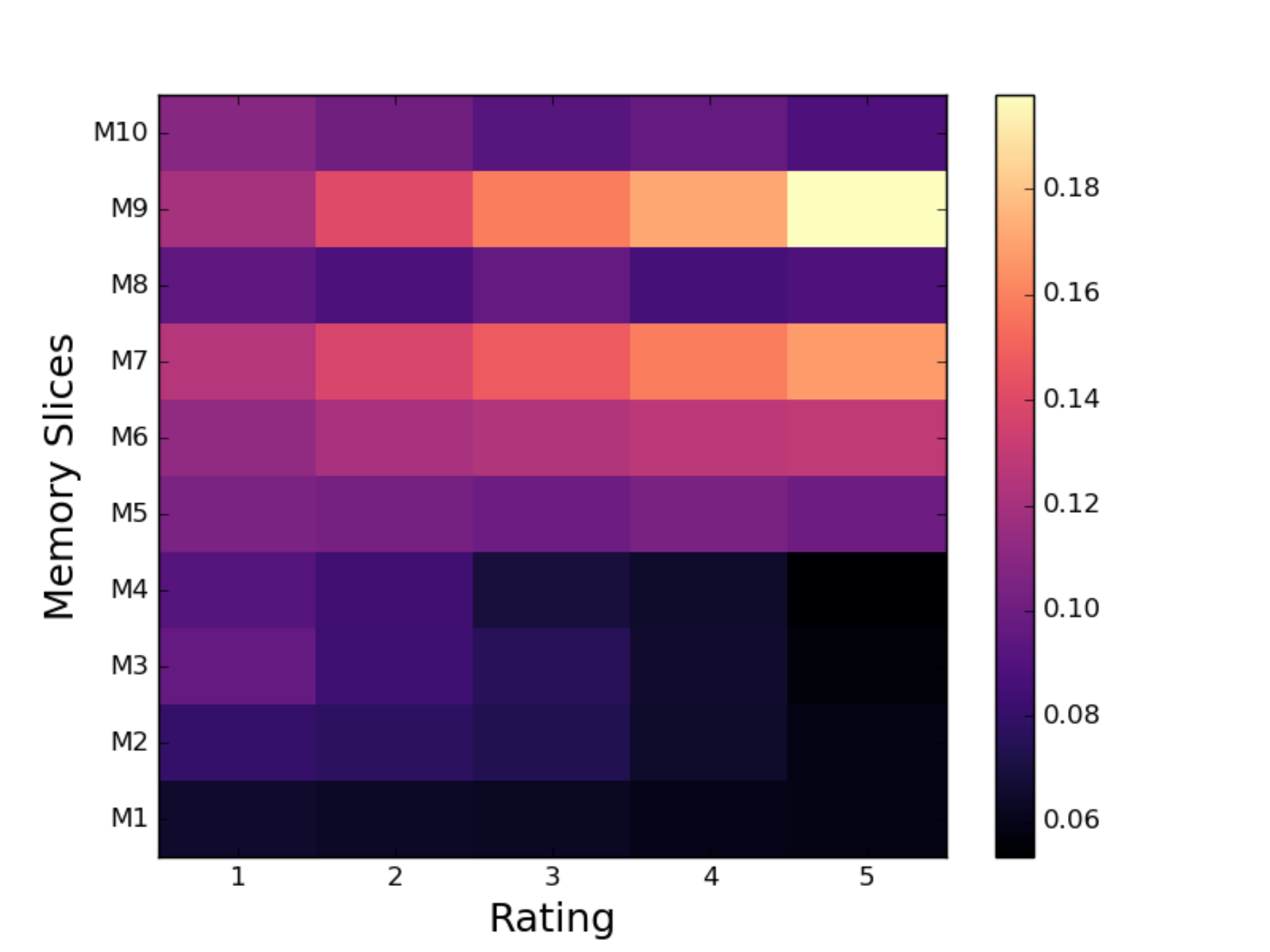}
\caption{Attention weights over LRAM for user-item pairs of different ratings on MovieLens1M. \textsc{LRML} is able to model explicit rating information
despite being only trained on implicit data. \textit{(Best viewed in color.)} }
\label{attention_rating}
\end{center}
\end{figure}

 The color scale represents the strength of the attention weights and each column of Figure \ref{attention_rating} represents the mean attention vector for each rating class. As such, we are able to observe patterns and trends across different ratings by looking at the rows (from left to right). For example, the mean attention vector of rating=1 is the first vertical slice in Figure \ref{attention_rating} and the intensity denoted by the color scale represents the attention weights. 

 Clearly, we observe that there is a pattern between the explicit rating score and the memory slice in which \textsc{LRML} is looking at. We observe that slices M2-M4 are mostly associated with bad ratings (1-2 stars) while having a high attention weight over M6, M7 and M9 signifies a good rating (4-5 stars). Moreover, there is a correlation between how much the model looks at M6, M7 and M9 and the explicit rating score. As such, it seems we are able to infer explicit rating scores solely based on how much our model is looking at each memory slice. 

 We believe this can be explained as follows: The goal of \textsc{LRML} is to find a latent relational structure between the user and item interactions. As such, while \textsc{LRML} is trying to assign relations between users and items via neural attention, it has learned to identify and model explicit rating sentiment from the implicit structure of the dataset.

\subsubsection{LRAM encodes temporal information}
The second discovery is that the LRAM module actually encodes temporal information. Similar to ratings, timestamps are available on the MovieLens1M dataset but are not used to train the model. To facilitate clear visualization, we binned the timestamps into 10 separate bins in ascending order. Figure \ref{fig:time} shows the visualized attention weights of LRAM with respect to time. Similar to explicit rating scores, we notice that certain memory slices model the chronological order of user-item interactions. On M8, we see that the intensity of the attention weights increase along with time, i.e., by viewing the row M8 from left to right, we can observe an increasing attention weight on M8 based on the intensity scale. Moreover, the converse is true for M6 which when observed from left to right, it decreases in intensity instead. In short, there is a clear pattern in which we can quite safely ascertain that LRAM has learned to encode temporal information. Once again, it is worthy to note that \textsc{LRML} was not given any temporal information to begin with. 

\begin{figure}[ht]
\begin{center}
\includegraphics[width=0.3\textwidth]{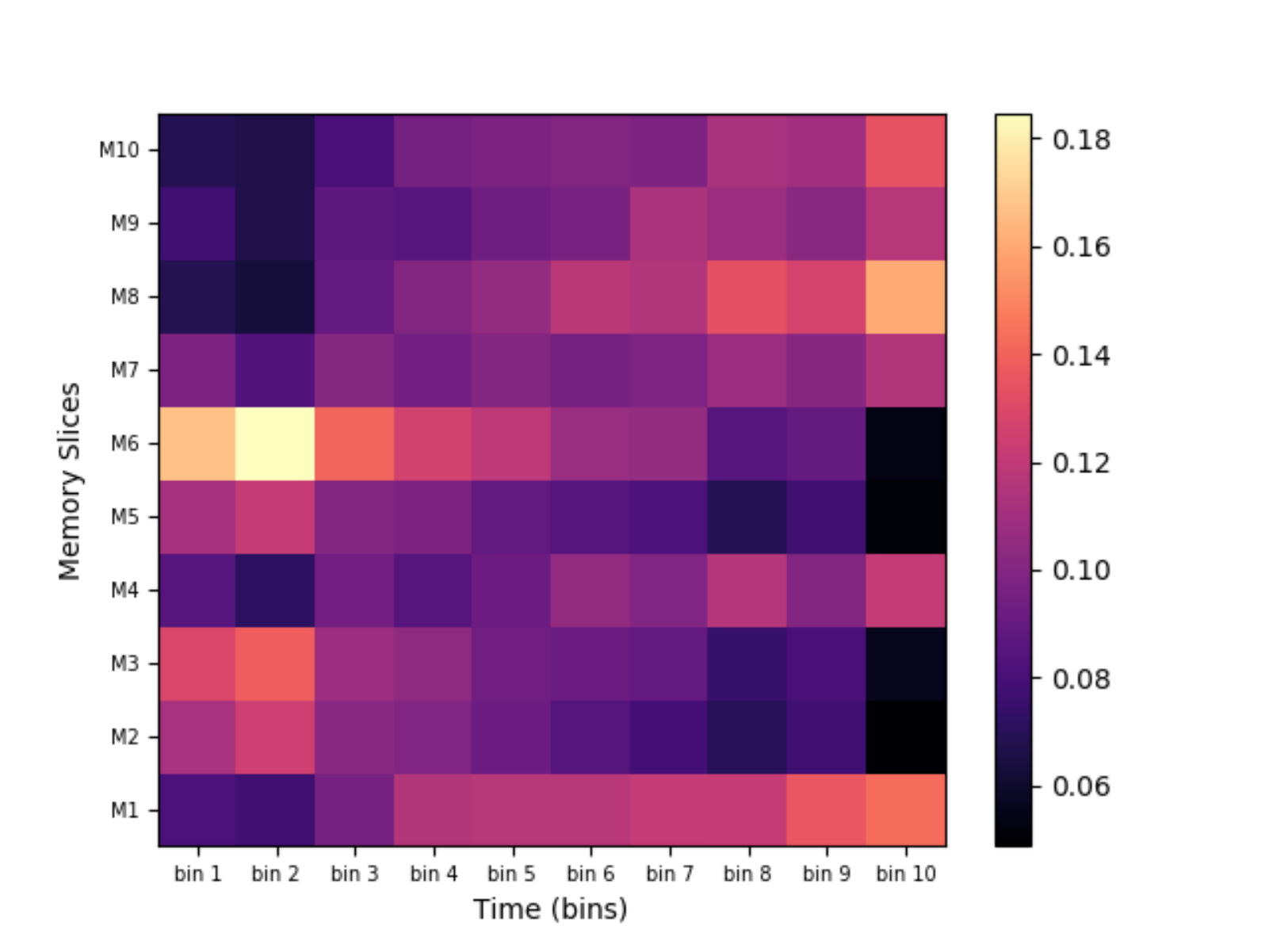}
\caption{Attention weights over LRAM for user-item pairs for different time bins on MovieLens1M. A clear trend is found in M6 which shows that the LRAM module encodes temporal information even when no such information is provided during training. \textit{(Best viewed in color.)} } 
\label{fig:time}
\end{center}
\end{figure}

\subsection{RQ5: What do the relation vectors represent? Are they meaningful?}
For each user-item pair in the test set, we generated the latent relation vector $r$. Next, we computed the cosine similarity between
the relation vectors of all user-item pairs and selected the user-item pairs in which the cosine similarity between their relation vectors
is the highest. Intuitively, this is to investigate \textbf{if similar user-item pairs might have similar relation vectors}. In order to characterize user-item pairs, we selected attributes that are available in the MovieLens1M dataset. The user attributes provided include Age, Job and Gender while only category and movie title were provided for the items. Once again, note that these attributes were not provided to our model during training. For each user-item pair, we computed the attribute matches with respect to the user-item pair with the \textbf{closest} relation vector, e.g., if $<$User1,Item1$>$ and $<$User2,Item2$>$ have the most similar relation vector, we compute the matches between \textbf{each} attribute of both user-item pairs. For example, we check for matches between User1 and User2 within the list of attributes such as user age, user gender and user job and item category. Ideally, the model should learn a similar relation vector for similar user-item pairs. In order to determine if the result is significant, we computed the probability of a match by random chance taking into consideration the distribution of attributes. Table \ref{tab:relation_match} reports the results of this experiment. 

\begin{table}[htbp]
  \centering
 \small
    \begin{tabular}{c|c|c|c}
    \midrule
    User-Item Attribute & Match ($\%$)  & Random ($\%$) & Diff (\%)\\
    \midrule
    User Age   &   25.81    & 22.17 &  3.64\\
    User Job   &  20.06     & 13.71 & 6.35\\
    User Gender & 65.73      & 59.43 & 6.30\\
    Item Category & 51.19      & 43.87 & 7.32 \\
    Category AND Job &  15.07    & 5.56 & 9.51 \\
    \midrule
    \end{tabular}%
     \caption{Matches between user-item attributes of user-item pairs with the closest relation vector. Relation vectors encode user-item attributes without being trained on them.}
  \label{tab:relation_match}%
\end{table}%
We observe that the percentage of getting an attribute match is often higher than that of random chance which might signal that \textit{similar user-item pairs have similar relation vectors}. In particular, the item category (movie genre) has the most prominent improvement over random chance ($7.32\%$) individually while a considerable percentage of user-item pairs (15.07\%) have an exact match of item category \textbf{and} job. This is $9.51\%$ more than random chance. Additionally, we also found that (by manual inspection) there is a prominent number of job-category matches such as \textit{(programmer, thriller)} and \textit{(technician/engineer, thriller)}. This is intuitive since engineers and programmers can be considered as semantically related professions. 

Overall, we believe that, the user and movie attributes characterize the behavior of users and therefore, there might be a hidden structure within simple implicit interaction data. By imposing and inducing architectural bias, our model learns to capture this fine-grained behavior even from simple implicit feedback data.

\section{Conclusion}
In this paper, we proposed \textsc{LRML} (Latent Relational Metric Learning), a novel attention-based memory-augmented neural architecture that models
the relationship between users and items in metric space using latent relation vectors. \textsc{LRML} demonstrates the state-of-the-art performance on 10 publicly available benchmark datasets for implicit collaborative ranking. Empirical results show that relative improvement is significantly greater when the dataset is large, e.g., Netflix Prize and MovieLens20M, which is due to the geometric inflexibility of the CML algorithm. Additionally, \textsc{LRML} leverages the hidden and latent relational structure in the implicit user-item interaction matrix. Via qualitative analysis of the attention weights, we discovered that explicit rating information, temporal information and even item attributes are encoded within the LRAM module and relation vectors even when these information are not provided during training. 

\section{Acknowledgements}
The authors would like to thank anonymous reviewers of WWW 2018 for their time and effort in reviewing this paper. We also thank Matt Yang and Kang KyungPhil for feedback on some typo errors in previous preprint versions.

\bibliographystyle{ACM-Reference-Format}

\bibliography{references}

\end{document}